\documentclass[11pt]{article}

\usepackage{amsthm}

\newtheorem{theorem}{Theorem}
\newtheorem{lemma}{Lemma}
\newtheorem{definition}{Definition}

\usepackage[utf8]{inputenc}
\usepackage[T1]{fontenc}
\usepackage{lmodern}

\usepackage{amsmath, amssymb}
\usepackage{bm}
\usepackage{mathtools}

\usepackage{graphicx}
\usepackage{booktabs}
\usepackage{array}
\usepackage{float}
\usepackage[font=small,labelfont=bf]{caption}

\usepackage{geometry}
\usepackage{tcolorbox}

\usepackage{abstract}

\usepackage[numbers,sort&compress]{natbib}

\usepackage{titlesec}
\titleformat{\section}{\large\bfseries}{\thesection.}{0.5em}{}
\titleformat{\subsection}{\normalsize\bfseries}{\thesubsection}{0.5em}{}
\titlespacing*{\section}{0pt}{1.5ex plus .2ex}{1ex}
\titlespacing*{\subsection}{0pt}{1.25ex plus .2ex}{0.75ex}

\usepackage[colorlinks=true,
            linkcolor=blue,
            citecolor=blue,
            urlcolor=blue]{hyperref}

\usepackage{titling}
\title{\textbf{Stream Neural Networks:\\
Epoch-Free Learning with Persistent Temporal State}\thanks{\textbf{Code availability:} Reference implementation at \href{https://github.com/pathan-amama/StNN-Stream-Neural-Networks}{https://github.com/pathan-amama/StNN-Stream-Neural-Networks}.}}

\author{
AMAMA PATHAN\\
\vspace{0.3em}
\textit{BLACKLOOP Research (Independent Research)}\\
\texttt{amamapathan@proton.me}
}

\date{2026}

\begin{document}
\maketitle


\begin{abstract}
Most contemporary neural learning systems rely on epoch-based optimization and repeated access to historical data, implicitly assuming reversible computation. In contrast, real-world environments often present information as irreversible streams, where inputs cannot be replayed or revisited. Under such conditions, conventional architectures degrade into reactive filters lacking long-horizon coherence.

This paper introduces \textit{Stream Neural Networks (StNN)}, an execution paradigm designed for irreversible input streams. StNN operates through a stream-native execution algorithm, the \textit{Stream Network Algorithm (SNA)}, whose fundamental unit is the \textit{stream neuron}. Each stream neuron maintains a persistent temporal state that evolves continuously across inputs.

We formally establish three structural guarantees: (1) stateless mappings collapse under irreversibility and cannot encode temporal dependencies; (2) persistent state dynamics remain bounded under mild activation constraints; and (3) the state transition operator is contractive for $\lambda < 1$, ensuring stable long-horizon execution. 

Empirical phase-space analysis and continuous tracking experiments validate these theoretical results. Together, these findings demonstrate that persistent temporal state is not an optimization choice but a structural requirement for neural systems operating under irreversible data flow.

The execution principles introduced in this work form the foundational substrate upon which subsequent learning, memory, monitoring, enforcement, and containment systems in the StNN framework are constructed.
\end{abstract}

\section{Introduction}

Most neural learning systems are designed around epoch-based training, where a finite dataset is repeatedly revisited to optimize model parameters. This design implicitly assumes reversible computation: past inputs remain accessible, gradients can be recomputed, and errors can be corrected through replay. While effective in offline settings, these assumptions break down in environments where data arrives as an irreversible stream.

In many real-world systems—such as continuous sensing, control, interaction, or autonomous operation—inputs are transient and cannot be replayed. Under such conditions, conventional neural architectures degrade into reactive mappings that respond only to the current input, lacking the ability to accumulate and preserve temporal structure over extended horizons.

A common workaround is to approximate temporal continuity through sliding windows, truncated histories, or external memory buffers. However, these approaches reintroduce batch-like behavior, incur unbounded computational cost, or rely on assumptions about data availability that are incompatible with irreversible streams.

This work takes a different approach. Rather than adapting batch-based learning to stream environments, we reconsider the execution model itself. We ask a more fundamental question: \emph{what execution properties are required for a neural system to remain stable under irreversible input streams?}

We argue that persistent temporal state is not an optimization choice but a structural requirement. Without internal state that evolves continuously across inputs, a neural system cannot sustain coherent dynamics when data cannot be revisited. This paper introduces a stream-native execution paradigm that embeds temporal persistence directly into the neuron-level computation, eliminating the need for epochs, batches, or replay.

The focus of this paper is strictly on execution dynamics. We provide formal structural guarantees establishing the necessity, boundedness, and contractive stability of persistent temporal state under irreversible execution. Learning mechanisms, memory management strategies, and stability-enhanced execution variants are intentionally deferred to subsequent work.

\section{Stream Neural Networks}

\subsection{Definition of Stream Neural Networks}

A Stream Neural Network (StNN) is a neural execution system designed to operate under irreversible input streams, where each input is processed exactly once and cannot be revisited. Unlike conventional architectures that rely on dataset reuse or epoch-based iteration, StNN defines computation as a continuous, forward-only process.

Execution in an StNN is governed by the \textit{Stream Network Algorithm (SNA)}. SNA specifies how signals propagate through the network over time, how internal temporal state is maintained, and how outputs are produced at each step of the stream. Importantly, SNA defines execution semantics only; it does not prescribe learning rules, memory policies, or control logic.

Under SNA, computation proceeds as a sequence of time-indexed transformations:
\[
(x_t, s_{t-1}) \;\rightarrow\; (y_t, s_t),
\]
where $x_t$ denotes the current input, $s_{t-1}$ the prior internal state, $y_t$ the instantaneous output, and $s_t$ the updated state. Each step consumes the input irreversibly and advances the system state forward in time.

This execution model differs fundamentally from batch-based neural computation. There is no notion of epochs, mini-batches, or gradient replay. Past inputs influence future behavior only through the evolution of internal state, not through repeated access to stored data.

The StNN framework separates execution from higher-level system functions. Learning algorithms, memory management mechanisms, monitoring processes, enforcement logic, and external containment systems may operate on or around SNA-defined execution, but they do not alter its temporal semantics. As a result, all subsequent systems built within the StNN framework inherit a common stream-native execution substrate without redefining how time and irreversibility are handled.

Parameter updates, when present, are constrained to operate strictly on the forward stream without violating state continuity or introducing replay.

\subsection{Formal Definition of Irreversibility}

\begin{definition}[Irreversible Input Stream]
Let $\{x_0, x_1, \dots, x_t, \dots\}$ be a sequence of inputs.
An execution process is said to operate under \emph{irreversibility} if, for all $t$:

\begin{enumerate}
    \item Each input $x_t$ is processed exactly once.
    \item No past input $\{x_0, \dots, x_{t-1}\}$ is stored, replayed, or reconstructable.
    \item State evolution depends only on $(x_t, s_{t-1})$.
\end{enumerate}

Formally, for any $k > 0$:
\[
x_{t-k} \notin \mathcal{M}_t
\]
where $\mathcal{M}_t$ denotes the accessible memory at time $t$.
\end{definition}

\subsection{Relation to Spiking Neural Networks}

The term ``Stream Neural Network (StNN)'' may invite comparison with spiking neural networks due to the shared emphasis on temporal behavior. However, the proposed framework is fundamentally different in both representation and execution semantics.

Unlike spiking neural networks, which encode information through discrete spike events and explicit spike timing, the proposed StNN framework operates through continuous temporal state propagation across discrete stream steps under irreversible input streams. Temporal information is preserved implicitly via persistent internal state rather than through event-based signaling.

The system does not model biological spiking behavior, does not rely on spike timing or rate codes, and does not employ spike-driven learning rules such as spike-timing-dependent plasticity. Instead, StNN is designed as a stream-native execution framework in which persistent internal state is a structural requirement for stable operation under forward-only data flow.

Accordingly, although both approaches address temporal processing, they differ in motivation, mechanics, and failure modes. The StNN framework should therefore be understood as a persistent-state stream execution model rather than a spiking neural system.

\section{Stream Neuron Architecture}

The fundamental computational unit of a Stream Neural Network is the \textit{stream neuron}. Unlike conventional neurons that operate statelessly on individual inputs, a stream neuron maintains an internal temporal state that persists across time steps. This state enables the neuron to accumulate, retain, and evolve information under irreversible input streams.

Figure~\ref{fig:stream_neuron} illustrates the internal execution architecture of a stream neuron. The diagram highlights how current inputs are integrated with persistent temporal state, without reliance on external memory buffers or replay mechanisms.

\begin{figure}[h]
    \centering
    \includegraphics[width=0.9\linewidth]{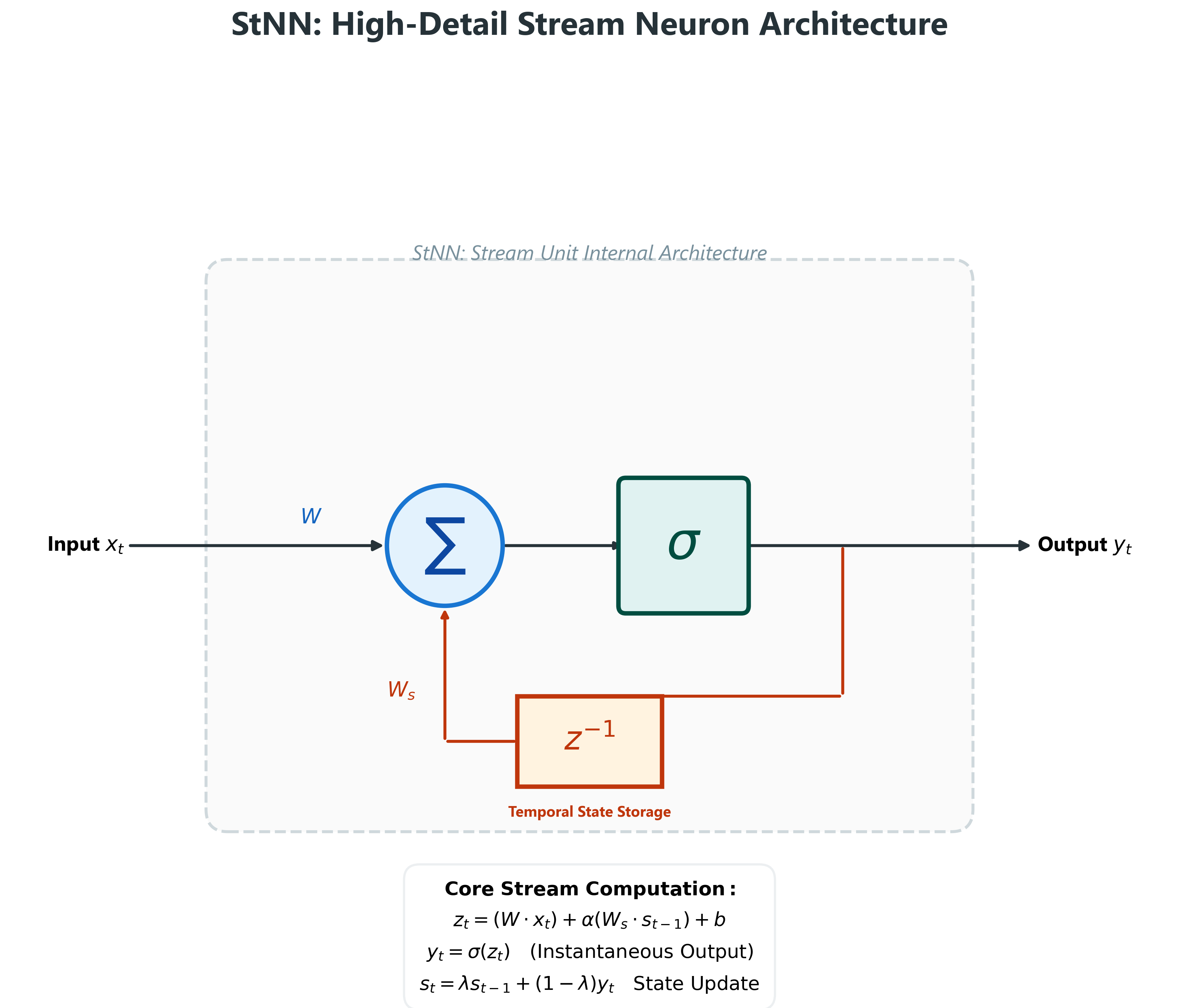}
    \caption{High-detail stream neuron architecture. Each neuron integrates the current input with a persistent temporal state stored across time steps, enabling irreversible stream execution without replay.}
    \label{fig:stream_neuron}
\end{figure}

At each time step $t$, a stream neuron receives an input vector $x_t$ and updates its internal state $s_t$ according to:
\begin{align}
z_t &= W x_t + \alpha W_s s_{t-1} + b, \\
y_t &= \sigma(z_t), \\
s_t &= \lambda s_{t-1} + (1 - \lambda) y_t,
\end{align}
where $W$ and $W_s$ are weight matrices, $b$ is a bias term, $\sigma(\cdot)$ is a nonlinear activation function, and $\lambda \in [0,1)$ controls temporal retention.

The internal state $s_t$ serves as the sole carrier of historical information. No past inputs are stored or revisited; temporal influence arises exclusively through state propagation. As a result, computation proceeds in constant time per neuron per step and advances irreversibly with the input stream.

The decay parameter $\lambda$ regulates the balance between responsiveness and persistence. Smaller values emphasize recent inputs, while larger values preserve long-horizon temporal influence. Crucially, this mechanism does not introduce external memory buffers or replay mechanisms. Temporal coherence emerges from the neuron’s internal dynamics rather than from dataset reuse.

The stream neuron architecture defines execution behavior only. Learning rules, memory management strategies, and stability enhancements operate independently of the neuron’s temporal update mechanism and are addressed in subsequent work. By embedding temporal persistence at the neuron level, StNN establishes a minimal and composable execution primitive suitable for continuous, irreversible computation.

\section{Structural Necessity of Persistent State}

\begin{theorem}[Stateless Collapse Under Irreversibility]
Let a stateless model be defined as:
\[
y_t = f(x_t; \theta)
\]
where $\theta$ is time-invariant and does not encode past inputs.

Under irreversible execution, the model cannot encode temporal dependencies beyond the current time step.
\end{theorem}

\begin{proof}
Since $y_t$ depends only on $x_t$, for any $k > 0$:
\[
\frac{\partial y_t}{\partial x_{t-k}} = 0.
\]

Thus the output is conditionally independent of past inputs given $x_t$:
\[
P(y_t \mid x_t, x_{t-k}) = P(y_t \mid x_t).
\]

Therefore,
\[
I(y_t; x_{t-k} \mid x_t) = 0.
\]

The model is a memoryless filter and therefore cannot encode temporal dependencies under irreversible execution.

\end{proof}

\section{Boundedness of Persistent Dynamics}

\begin{theorem}[Bounded State Dynamics]
Let the state update be defined as:
\[
s_t = \lambda s_{t-1} + (1-\lambda) y_t,
\]
where $\lambda \in [0,1)$ and $|y_t| \le M$ for all $t$.

If $|s_0| \le M$, then $|s_t| \le M$ for all $t \ge 0$.
\end{theorem}

\begin{proof}
Assume $|s_{t-1}| \le M$.

By triangle inequality:
\[
|s_t| \le \lambda |s_{t-1}| + (1-\lambda)|y_t|.
\]

Substituting bounds:
\[
|s_t| \le \lambda M + (1-\lambda)M = M.
\]

Thus the bound holds for all $t$ by induction.
\end{proof}

\begin{lemma}[Contraction Property]
Let two state trajectories evolve under identical input sequence $\{y_t\}$ with
\[
s_t = \lambda s_{t-1} + (1-\lambda) y_t,
\]
where $\lambda \in [0,1)$.

Then the state transition map is a contraction with respect to the previous state:
\[
\| s_t - s'_t \| \le \lambda \| s_{t-1} - s'_{t-1} \|.
\]
\end{lemma}

\begin{proof}
Consider two states $s_{t-1}, s'_{t-1}$ under identical input $y_t$:
\[
s_t - s'_t = \lambda (s_{t-1} - s'_{t-1}).
\]

Taking norm:
\[
\lVert s_t - s'_t \rVert = \lambda \lVert s_{t-1} - s'_{t-1} \rVert.
\]

Since $\lambda < 1$, the mapping is a contraction.
\end{proof}

\begin{tcolorbox}[colback=gray!5,colframe=black,title=Formal Guarantees Summary]

Under irreversible execution, Stream Neural Networks satisfy the following structural properties:

\begin{enumerate}
    \item \textbf{Necessity of Persistent State:} Stateless mappings collapse to memoryless filters and cannot encode temporal dependencies (Theorem~1).
    \item \textbf{Bounded State Dynamics:} Persistent state evolution remains confined within a bounded attractor region when activations are bounded (Theorem~2).
    \item \textbf{Contractive Stability:} The state transition operator is contractive for $\lambda < 1$, ensuring stable long-horizon execution.
\end{enumerate}

These guarantees are structural and independent of learning rules, memory governance policies, or optimization procedures.

\end{tcolorbox}

\section{Phase Space Dynamics of Stream Execution}

To examine the role of persistent temporal state in stream execution, we analyze the system’s behavior in phase space. Rather than evaluating predictive accuracy, this analysis focuses on qualitative stability properties under irreversible input streams.

Figure~\ref{fig:phase_space} compares phase space trajectories of a stream neuron under two execution regimes: with temporal state enabled and with temporal state disabled. When state is enabled, execution converges to a stable limit cycle, indicating sustained temporal dynamics. In contrast, disabling temporal state causes trajectories to collapse into a point attractor.

\begin{figure}[H]
    \centering
    \includegraphics[width=0.7\linewidth]{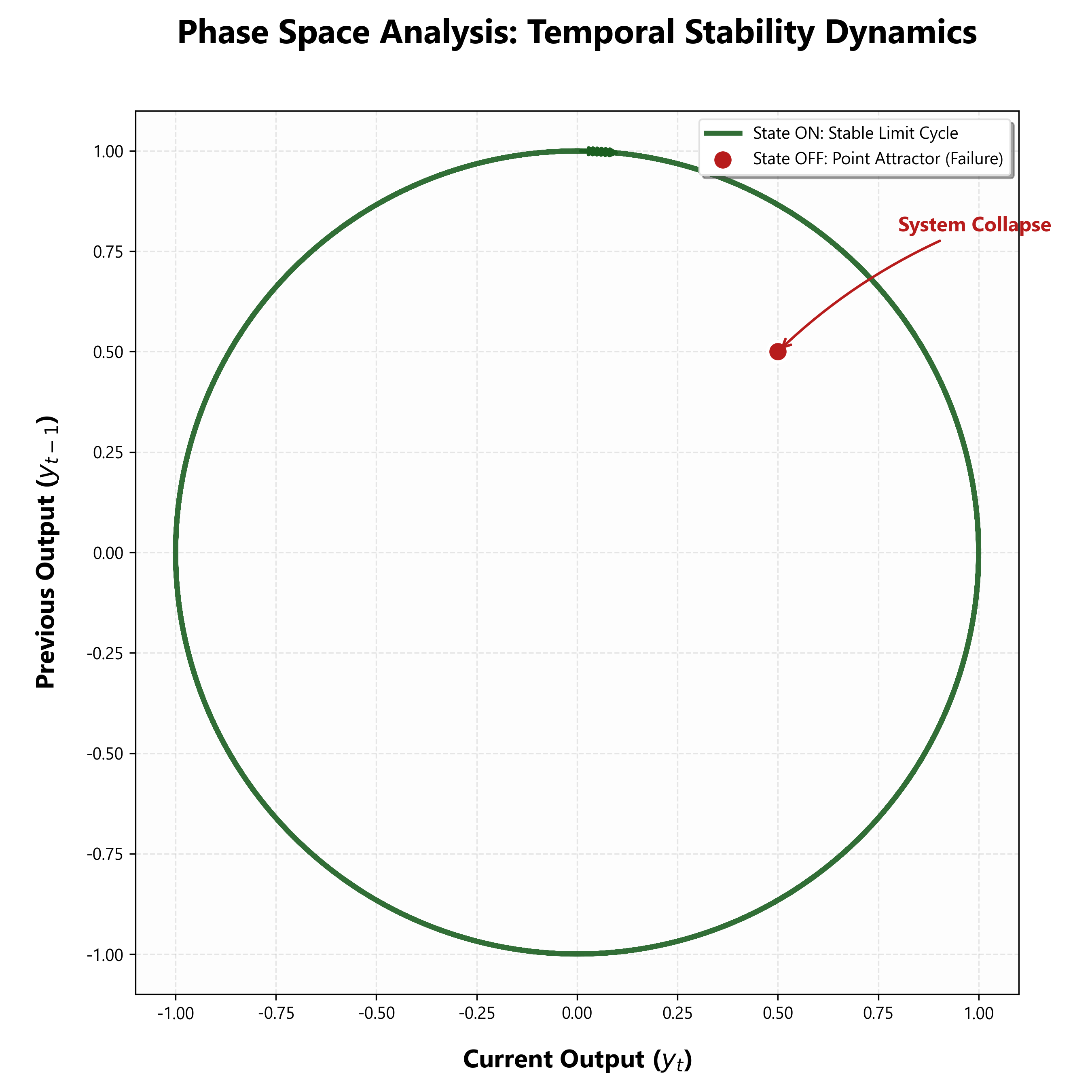}
    \caption{Phase space analysis of stream execution. State-enabled dynamics converge to a stable limit cycle, while state-disabled execution collapses to a point attractor, indicating loss of temporal coherence.}
    \label{fig:phase_space}
\end{figure}

The point attractor observed in state-disabled execution represents a structural failure mode. In accordance with Theorem~1, removing persistent state reduces execution to a memoryless mapping $y_t = f(x_t; \theta)$ under irreversibility. Such a mapping cannot encode temporal dependencies and therefore collapses to a reactive filter. This behavior reflects a fundamental incompatibility between stateless execution and irreversible data streams.

By contrast, the emergence of a limit cycle under state-enabled execution indicates bounded yet persistent dynamics consistent with Theorem~2. The system neither diverges nor collapses, allowing temporal structure to be maintained indefinitely without replay or external memory. Importantly, this behavior arises solely from execution dynamics and does not depend on learning rules, optimization procedures, or parameter adaptation.

These phase space properties demonstrate that persistent temporal state is not an optional enhancement but a structural requirement for stable stream execution. Any neural system operating under irreversible input streams must preserve internal state across time to avoid dynamical collapse.

\section{Temporal Retention Dynamics}

Persistent temporal state in a stream neuron is regulated through a decay mechanism that controls how past information influences future execution. This mechanism enables continuous adjustment between short-horizon responsiveness and long-horizon temporal coherence without introducing external memory buffers or replay.

Figure~\ref{fig:temporal_decay} illustrates the evolution of internal state retention under different decay factors $\lambda$. Smaller values of $\lambda$ result in rapid decay, emphasizing recent inputs, while larger values preserve historical influence over extended time horizons.

Because $\lambda < 1$, historical influence decays exponentially, ensuring bounded contribution from distant past states.

\begin{figure}[H]
    \centering
    \includegraphics[width=0.9\linewidth]{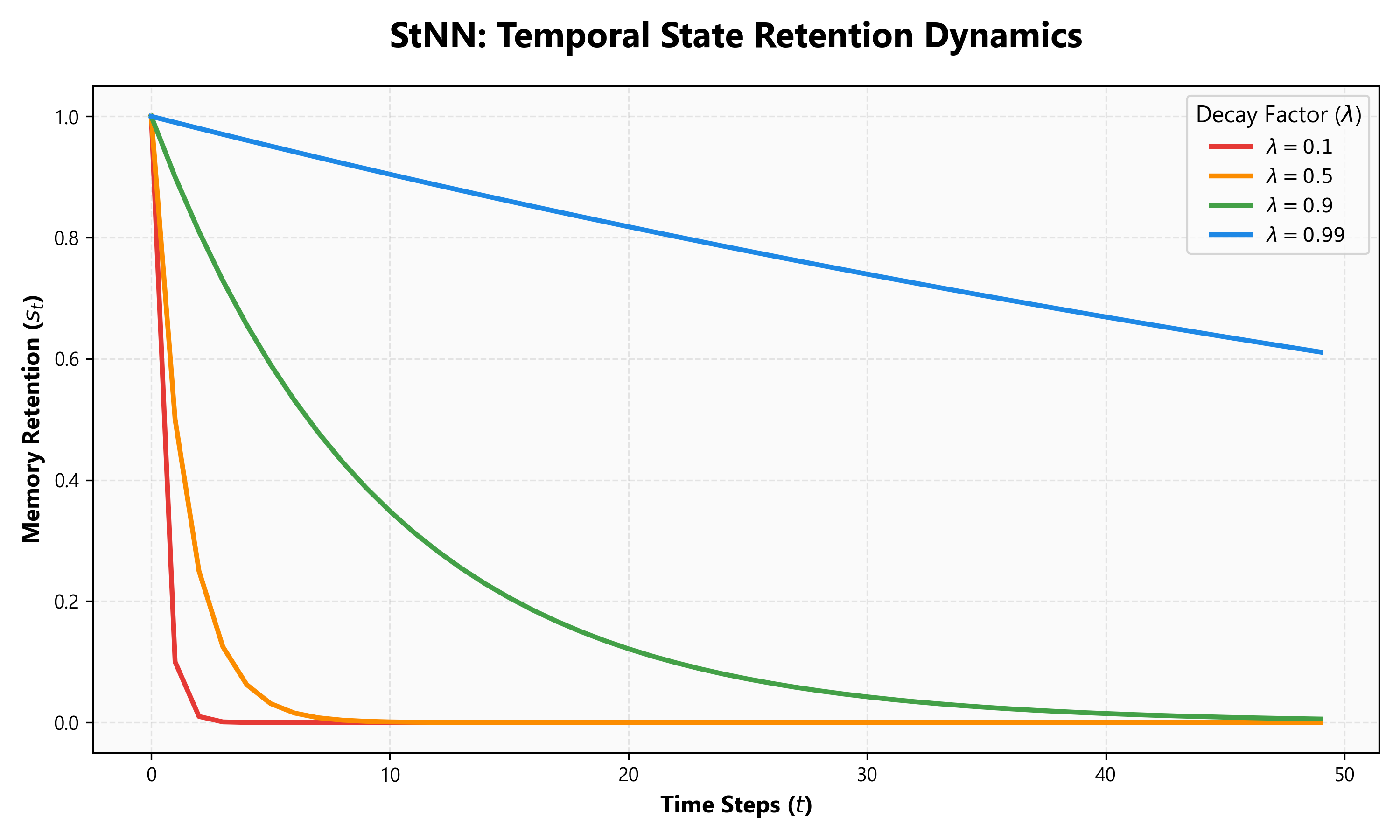}
    \caption{Temporal state retention under different decay factors $\lambda$. Larger values maintain long-horizon influence, while smaller values favor rapid responsiveness.}
    \label{fig:temporal_decay}
\end{figure}

Importantly, decay operates locally within each stream neuron and does not require storage of past inputs. Temporal influence is carried forward exclusively through the internal state variable, ensuring constant-time execution per step regardless of stream length.

The presence of a tunable decay parameter demonstrates that temporal persistence in StNN is neither fixed nor unbounded. Instead, it provides a controlled mechanism for regulating execution dynamics under irreversible input streams. This control is intrinsic to the execution model and does not rely on learning rules, memory management policies, or optimization procedures.

These dynamics establish temporal retention as a continuous execution property rather than a discrete architectural feature. As a result, stream execution can be adapted to diverse temporal regimes while preserving irreversibility and bounded per-step computation.

\section{Continuous Tracking Under Irreversible Input}

To evaluate execution behavior under sustained, irreversible input streams, we examine the ability of stream neurons to maintain continuous temporal tracking over time. This analysis focuses on execution coherence rather than predictive accuracy and characterizes baseline stream behavior under persistent temporal state.

Figure~\ref{fig:tracking} illustrates continuous tracking behavior under irreversible input. When temporal state is enabled, the stream neuron maintains coherence across time steps and follows the underlying signal without replay or buffering. For comparison, execution behavior without temporal state is shown to highlight the absence of sustained tracking capability.

\begin{figure}[H]
    \centering
    \includegraphics[width=0.9\linewidth]{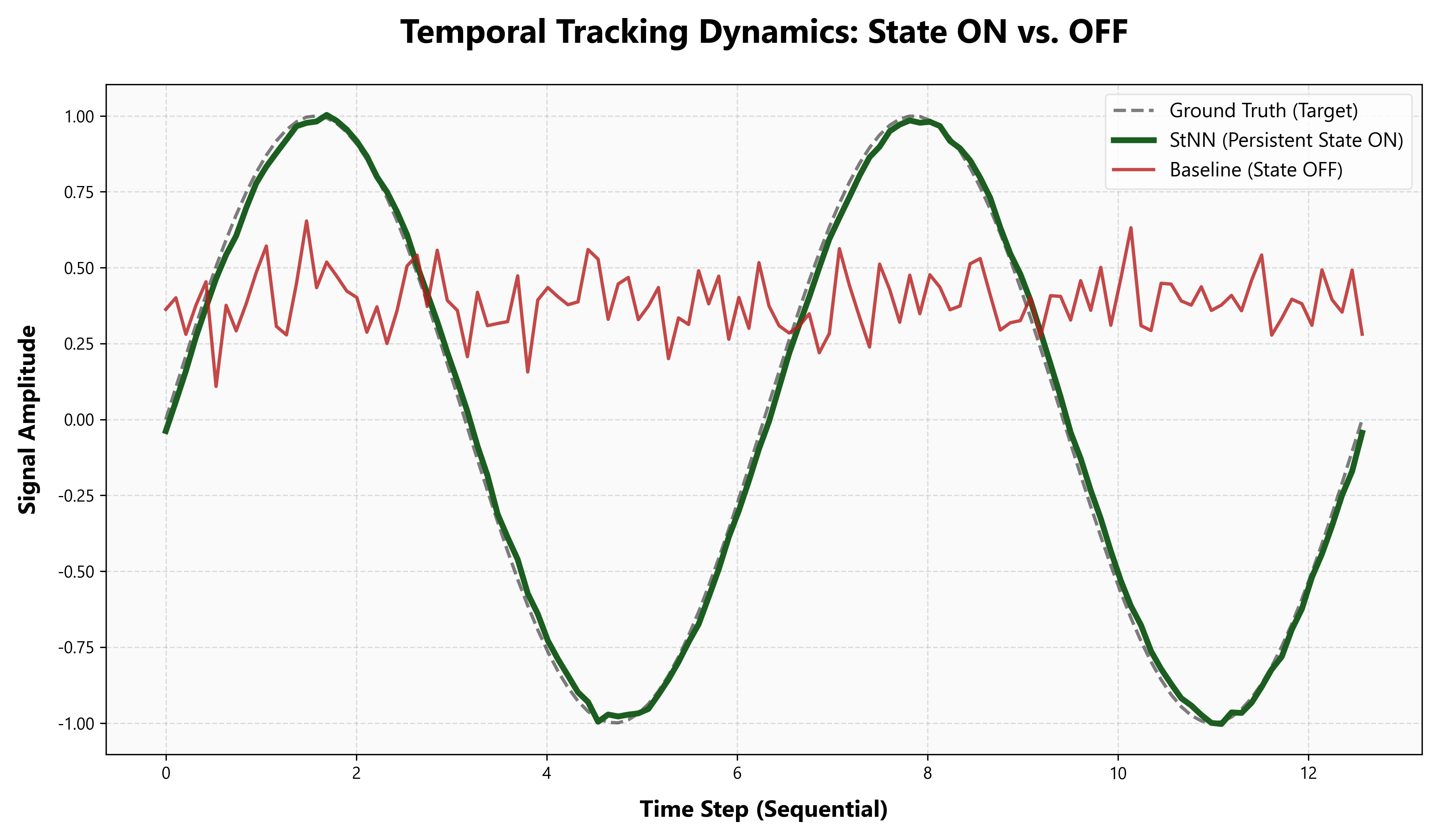}
    \caption{Baseline continuous tracking under irreversible input streams. State-enabled execution sustains coherent temporal dynamics, while execution without persistent state fails to preserve temporal structure.}
    \label{fig:tracking}
\end{figure}

The loss of tracking observed without temporal state is not attributable to insufficient capacity or parameter selection. Rather, it reflects the absence of a mechanism for accumulating temporal information across inputs. Without persistent state, each input is processed independently, preventing meaningful temporal integration under irreversible conditions.

By contrast, state-enabled execution propagates internal state forward in time, allowing coherent tracking behavior to be sustained indefinitely. This demonstrates that persistent temporal state is a foundational requirement for stable execution under continuous streams.

\section{Scope and Limitations}

The results presented herein establish structural guarantees on execution semantics under irreversibility, including bounded persistent state evolution and contraction of the state transition operator. No claims are made regarding learning convergence, statistical optimality, or bounded memory policies.

In particular, this paper does not specify learning algorithms, parameter update rules, or training procedures. While execution is defined through the Stream Network Algorithm (SNA), the mechanisms by which parameters are adjusted over time are deferred to subsequent work. Similarly, this paper does not introduce memory management policies, bounded storage systems, or pruning mechanisms beyond the neuron-local temporal state described herein.

Stability-enhanced execution variants, monitoring processes, enforcement logic, and external containment mechanisms are also outside the scope of this paper. Such systems may operate on or around the execution substrate established here, but they do not alter its temporal semantics. The purpose of this work is to define the minimal execution properties required for neural computation under irreversible streams, independent of higher-level system functions.

As a result, the conclusions of this paper should be interpreted as statements about execution structure rather than complete system behavior. Subsequent papers in the StNN framework extend this foundation by addressing learning, bounded memory, monitoring, enforcement, and containment while preserving the execution constraints established here.

\section{Conclusion}

This paper introduced Stream Neural Networks (StNN) as an execution paradigm for neural systems operating under irreversible input streams. By embedding persistent temporal state directly into the neuron-level computation, StNN eliminates reliance on epochs, batch processing, and replay while maintaining stable execution over unbounded time horizons.

Through phase space analysis, temporal retention dynamics, and continuous tracking experiments, we demonstrated that persistent internal state is not an optimization choice but a structural requirement. Disabling temporal state causes execution to collapse into a memoryless regime incompatible with irreversible data flow, whereas state-enabled execution exhibits bounded, coherent dynamics without external memory or replay.

The results presented here establish the execution substrate of the StNN framework. Subsequent work builds upon this foundation by addressing complementary aspects of stream-native systems. StNN-1 empirically isolates the necessity of temporal state through controlled ablation studies. StNN-2 completes the execution framework by introducing bounded learning and memory mechanisms while preserving the stream execution semantics defined in this paper.

Higher-level systems for monitoring, enforcement, and containment developed in subsequent work

These systems inherit the irreversible temporal semantics of StNN without modifying its core execution dynamics.
Accordingly, the present work defines the minimal execution principles upon which more complex stream-native systems can be constructed.

\end{document}